\documentclass[a4paper, twocolumn, 10pt]{article}
\usepackage[vmargin=2cm, hmargin=2.0cm]{geometry}
\usepackage{graphicx} 
\usepackage{times} 
\usepackage{authblk} 
\usepackage{xcolor} 
\usepackage{sectsty}
\usepackage{diagbox}
\usepackage{authblk}
\usepackage{wrapfig}
\usepackage{graphicx}
\allsectionsfont{\raggedright} 
\sectionfont{\fontsize{10}{12}\selectfont}

\usepackage{titlesec}
\usepackage[round]{natbib}
\usepackage{microtype}

\usepackage{amsmath}
\usepackage{amsthm}
\usepackage{booktabs}
\usepackage{algorithm}
\usepackage{algorithmic}
\usepackage{multirow}

\newtheorem{theorem}{Theorem}
\newtheorem{lemma}{Lemma}

\usepackage[pdfsubject={Lastname F., Lastname2 F2. and Lastname3 F3.: Example Paper: a General Model of Information Transfer},
pdfauthor={Lastname F., Lastname2 F2. and Lastname3 F3.},
pdftitle={Lastname F., Lastname2 F2. and Lastname3 F3.: Example Paper: a General Model of Information Transfer},
pdfkeywords={Model; pi in the C-Ky; X-Y-Z Analysis}]{hyperref}

\titleformat{\subsubsection}{\normalfont\itshape}{\thesubsubsection}{1em}{}
\titleformat{\subsection}{\normalfont\bfseries}{\thesubsection}{1em}{}

\makeatletter
\renewcommand\@biblabel[1]{}

\renewcommand\newblock{\hskip .11em\@plus.33em\@minus.07em}
\makeatother

\title{\textls[-24]{
\fontsize{14}{14}\selectfont\lsstyle\textbf{
Deep Embedding Clustering Driven by Sample Stability
}}}

\author{Zhanwen Cheng}
\author{Feijiang Li}
\author{Jieting Wang}
\author{Yuhua Qian}
\affil{Institute of Big Data Science and Industry ,Shanxi University ,Taiyuan 030006 , China}

\date{}
\begin{document}

\maketitle


\begin{abstract}
    Deep clustering methods improve the performance of clustering tasks by jointly 
    optimizing deep representation learning and clustering. 
    While numerous deep clustering algorithms have been proposed, most of them rely 
    on artificially constructed pseudo targets for performing clustering. 
    This construction process requires some prior knowledge, 
    and it is challenging to determine a suitable pseudo target for clustering. 
    To address this issue, we propose a deep embedding clustering 
    algorithm driven by sample stability (DECS), which eliminates the requirement of 
    pseudo targets. 
    Specifically, we start by constructing the initial feature space with an autoencoder 
    and then learn the cluster-oriented embedding feature constrained by sample stability. 
    The sample stability aims to explore the deterministic relationship between samples 
    and all cluster centroids, pulling samples to their respective clusters and keeping 
    them away from other clusters with high determinacy. 
    We analyzed the convergence of the loss using Lipschitz continuity in theory, 
    which verifies the validity of the model. 
    The experimental results on five datasets illustrate that the proposed method achieves 
    superior performance compared to state-of-the-art clustering approaches.
\end{abstract}

\section{Introduction}
Clustering \cite{xu2005survey}, one of the most crucial tasks in machine learning, 
aims to group similar samples into the same cluster while separating dissimilar 
ones into different clusters.
Traditional clustering methods such as k-means \cite{macqueen1967some}, 
spectral clustering \cite{ng2001spectral,yang2018new}, 
Gaussian mixture model \cite{bishop2006pattern,reynolds2009gaussian} and 
hierarchical clustering \cite{sneath1962numerical,johnson1967hierarchical,koga2007fast} 
have achieved tremendous success over the past decades. 
However, these methods depend on manually extracted features, making them impractical for 
high-dimensional and unstructured data. 
Benefiting from the development of deep representation learning, deep clustering arises 
and has attracted increasing attention recently.

The existing deep clustering methods can be roughly categorized into three types: 
First, the pseudo labeling deep clustering method \cite{niu2020gatcluster,niu2022spice} 
filters out a subset of samples with high confidence and trains in a supervised manner, 
yet, the performance of this method heavily relies on the quality of the filtered pseudo labels, 
which is susceptible to model capability and hyper-parameter tuning. 
Second, the self-training deep clustering method \cite{xie2016unsupervised,guo2017improved} 
optimizes the distribution of cluster assignments by minimizing the KL-divergence between 
the assignment distribution and an auxiliary distribution, but the performance of this 
method is limited by the construction method of the auxiliary distribution. 
Third, the contrastive deep clustering method \cite{jaiswal2020survey,jing2020self} 
aims to pull the positive pairs close while pushing the negative pairs far away, this 
method relies on the construction approach of positive and negative sample pairs. 
In summary, despite the proposal of numerous excellent deep clustering methods, 
most of them rely on artificially constructed pseudo targets that require prior 
knowledge and may heavily impact the clustering results. 

In this paper, inspired by traditional clustering methods based on sample 
stability \cite{li2019clustering,li2020clustering}, 
we propose a deep embedding clustering algorithm driven by sample stability (DECS). 
Different from prior methods, our method eliminates the requirement of a pseudo target, 
and clustering using sample stability as a constraint. 


Specifically, our method consists of two stages: representation learning and clustering. 
In the representation learning stage, we employ a convolutional autoencoder 
\cite{guo2017deep} to map the raw data into a latent embedding space that 
captures the underlying structure of the data. 
This is achieved by minimizing the reconstruction loss, ensuring that the learned 
feature representations preserve the essential information from the original input. 
Subsequently, in the clustering stage, we retain the encoder module of the autoencoder 
and compute the soft assignment probabilities of each sample to all cluster centroids 
based on the learned embedding representations, which we refer to as co-association 
probability. 
Then, we compute the level of determinacy for each sample with respect to all cluster 
centers and further calculate the stability of all samples by considering their determinacy 
levels regarding each cluster center.
To the best of our knowledge, our method is the first to utilize the deterministic 
relationship between samples and centroids for clustering, and innovatively employ 
the instability of samples as a loss to optimize the parameters of deep neural networks.
In summary, the main contributions of this work are as follows:
\begin{itemize}
    \item The concept of sample stability is extended into deep clustering, and a novel 
            loss function that effectively captures both intra-class and inter-class 
            relationships among the samples is proposed. This approach is then applied in 
            a joint learning framework, which comprises an autoencoder and a clustering layer. 
    \item The convergence of the model is theoretically analyzed, providing 
            evidence that clustering with the internal sample relationship driven by sample 
            stability can indeed converge. 
    \item Experiments are conducted on five image datasets to validate the 
            effectiveness of our method. The experimental results demonstrate that our 
            approach outperforms the state-of-the-art methods.
\end{itemize}

\begin{figure*}[h]
    \centering
    \includegraphics[width=0.8\textwidth]{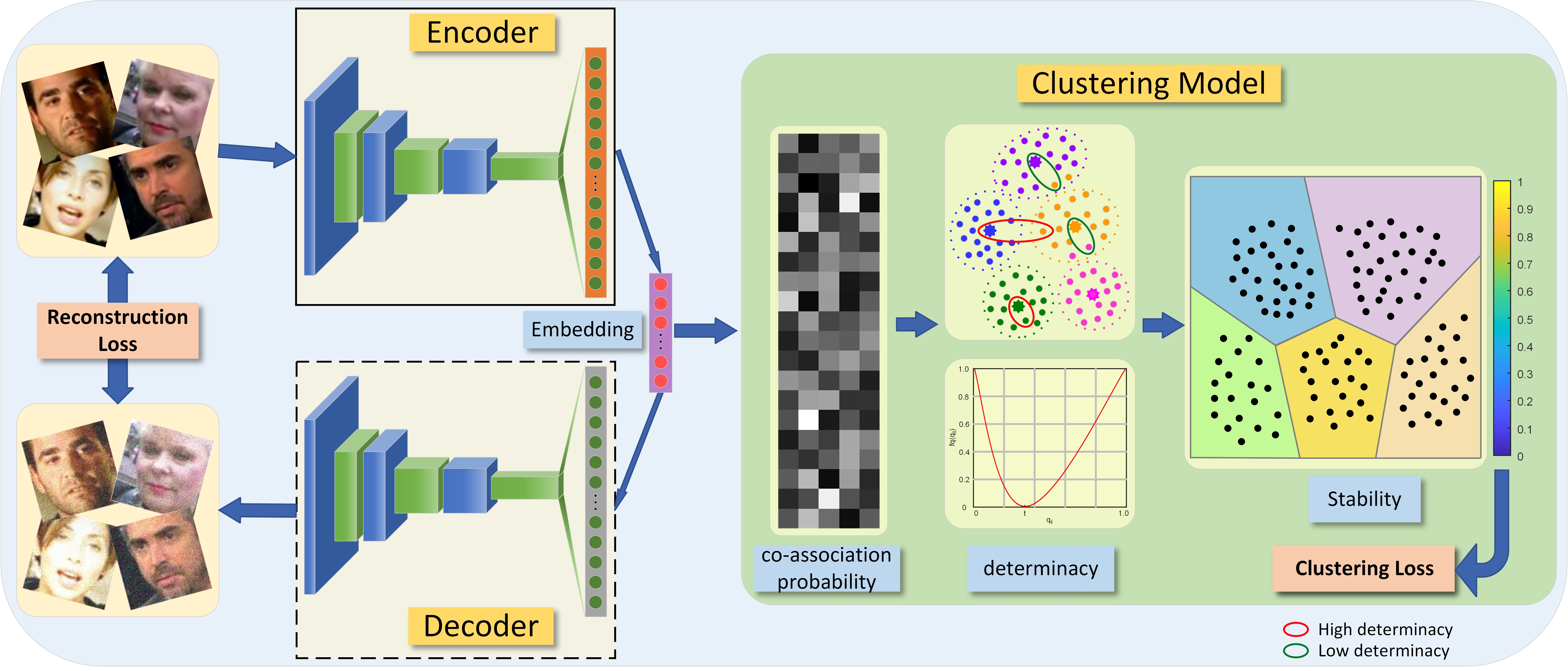}
    \caption{Pipeline of the proposed DECS. We first train an autoencoder consisting of 
            an encoder and a decoder to embed the inputs into a latent space and reconstruct 
            the input samples using their latent representations. The reconstruction loss 
            is utilized to learn discriminative information from the inputs. Then, we 
            discard the decoder and jointly optimize the encoder and clustering model 
            to get the clustering results.}
    \label{fig:stability}
\end{figure*}

\section{Related Work}
This work is closely related to convolutional autoencoder and sample stability, which are 
briefly introduced in this section. 

\subsection{Convolutional Autoencoder}
Autoencoder is an unsupervised neural network model widely used for tasks like data 
dimensionality reduction and feature extraction. 
Deep Embedded Clustering (DEC) \cite{xie2016unsupervised} was the pioneer in utilizing denoising 
autoencoders for joint learning of feature representations and cluster assignments. 
Subsequent works, such as IDEC \cite{guo2017improved}, FCDEC-DA \cite{guo2018deep}, 
SDEC \cite{ren2019semi}, and others, have built upon DEC's autoencoder framework, and achieved 
remarkable clustering results. 

Due to the limited ability of fully connected layers in extracting features from high-dimensional 
data, such as images, convolutional autoencoder \cite{guo2017deep} was proposed by incorporating 
Convolutional Neural Networks (CNNs) into autoencoders, which showed improved adaptation to 
image-related tasks. 
DEPICT \cite{ghasedi2017deep}, ConvDEC-DA \cite{guo2018deep}, DBC \cite{li2018discriminatively}, 
StatDEC \cite{rezaei2021learning}, and so on, 
have adopted convolutional autoencoders instead of fully connected autoencoders in order 
to learn feature representations and achieve superior clustering results.

\subsection{Sample Stability}
The concept of sample stability \cite{li2019clustering} was first proposed in 
clustering ensembles.
Clustering methods based on sample stability aim to explore the indeterminacy 
of sample relationships and identify sets of samples with stable relationships. 
These approaches leverage pairwise relationships between samples for clustering, 
which reduces the impact of indeterminate relationships among samples. 

Given a set of clustering results, the co-association probability between two samples can be 
represented by their frequency of belonging to the same cluster based on their similarity. 
A co-association probability value of one indicates high determinacy that 
the samples belong to the same cluster, while a value of zero indicates high determinacy that 
they do not belong to the same cluster. 
However, when the value falls between zero and one, it becomes difficult to definitively 
determine whether the two samples belong to the same cluster, resulting in a low determinacy.
Due to the insufficiency of using co-association probabilities alone in reflecting 
the level of determinacy regarding whether two samples belong to the same cluster, the 
determinacy function \cite{li2019clustering} was proposed to evaluate the level of 
determinacy between two samples. 
Then, the stability of sample $x_i$ is defined as the average level of determinacy 
among sample $x_i$ and all other samples based on their determinacy.

Subsequently, SSC \cite{li2020clustering} extended the concept of sample stability from 
cluster ensembles to clustering analysis and proposed a new function for measuring sample 
stability in cluster analysis, and the theoretical validity of sample stability was 
analyzed in this work.

\section{Method}
In this section, we present the proposed 
Deep Embedding Clustering Driven by Sample Stability (DECS) model. 
Our model first trains a convolutional autoencoder and then utilizes sample stability 
as guidance to accomplish clustering. Fig.\ref{fig:stability} illustrates the 
overall process.

\subsection{Problem Formulation}
In this paper, we aim to cluster a set of $N$ samples $X = \left \{ x_{i}  \right \} _{i=1}^{N}$ 
from the input space $X \in R^{d}$ into $K$ classes using a clustering network. 
Distinguishing with the prior works, we reconsider the problem of clustering in deep 
neural networks by introducing constraints on the relationships among samples, 
and make the first attempt to reduce the calculation of sample stability from $N^2$ to $kN$. 
Our method first employs an autoencoder to map the sample set $X$ into a representation space, 
and then utilizes sample stability as a constraint to achieve sample clustering.



To this end, the objective function of our framework can be formulated as:
\begin{equation}
    L= L_{r} + L_{c},
\end{equation}
where $L_{r}$ and $L_{c}$ represent the reconstruction loss and clustering loss, 
respectively. 

\subsection{Extract Features with Convolutional Autoencoder}

To accurately represent each sample with an embedding vector, we employ a 
convolutional autoencoder as the feature extractor, 
an encoder $f_{\theta _{e}}$ is used to map a sample $x_i \in X$ to its latent embedding 
vector $z_i \in Z$, while a decoder $g_{\theta _{d}}$ reconstructs $x_i$ from its 
embedding vector $z_i$. 
To be specific, given a set of samples $X = \left \{ x_{i} \in R^{D} \right \}_{i=1}^{n}$, 
a random transformation $T_{random}$ is applied to each sample $x_i$ to obtain the augmented 
sample $x_{i}'=T_{random}\left ( x_{i} \right )$, 
and then serve as the inputs of convolutional autoencoder, which extracts 
the latent embedding representation of each sample from its internal layers by minimizing 
the reconstruction loss: 
\begin{equation}
    L_{r}=\frac{1}{n} \sum_{i=1}^{n}\left\|g_{\theta _{d}}\left(f_{\theta _{e}}\left(x_{i}'\right)\right)-x_{i}' \right\|_{2}^{2}, 
    \label{con:Lr}
\end{equation}
where $n$ represent the number of samples.

The convolutional encoder is utilized to capture the essential information of 
the samples, while a convolutional decoder is employed to validate and enhance the representation 
ability of the embedding vectors. This process can be expressed as:
\begin{equation}
    f_{\theta _{e}} =\sigma \left ( \sum_{i\in H}x'*W_{i}+b_{i} \right ),
\end{equation}
\begin{equation}
    g_{\theta _{d}} =\sigma \left ( \sum_{i\in H} z_{i} *\tilde{W}_{i}+c_{i} \right ),
\end{equation}
where $\theta _{e}$ and $\theta _{d}$ represent the parameters of the convolutional encoder and decoder, 
respectively, $\sigma$ is the activation function such as ReLU, $H$ denotes the group of latent 
feature maps, $W_{i}$ and $b_{i}$ correspond to the filter and bias of the $i^{th}$ feature map 
in the encoder, similarly, $\tilde{W}$ and $c_{i}$ are the corresponding parameter in the decoder, 
and $*$ denotes the convolution operation.

\subsection{Clustering with Sample Stability}
In the clustering stage, we utilize the encoder trained in the previous stage as the 
feature extractor and then fine-tune the encoder using sample stability as guidance, 
which ensures that it learns cluster-oriented sample representations. 
For the sake of writing and understanding, we consider only a single sample and describe 
the computational processes in vector form in the following description. 
We introduce the notation $\mathbf{I} \in R^{1\times k} $ as a row vector with all ones, 
and $T$ denotes the transpose of a vector. 

In the clustering stage, we first perform k-means clustering in the embedding space to obtain 
initial centroids 
$\mathbf{m} = \left [ m_{1},m_{2},\dots,m_{k} \right ] \in R^{k\times d}$, 
where $k$ and $d$ represent the number and dimension of centroids, respectively. 
Next, we calculate the assignment probability between a sample embedding $z_i$ and the 
centers $\mathbf{m}$ of all clusters by a Student's t-distribution:
\begin{equation}
    \mathbf{q_i} =
    \frac{ \mathbf{I} \left [ \mathbf{E} + \frac{1}{\alpha}diag\left ( \left ( \mathbf{Z_{i}}-\mathbf{m} \right ) \left ( \mathbf{Z_{i}}-\mathbf{m} \right )^{T} \right )  \right ]^{-\frac{\alpha + 1}{2}}}
    {\mathbf{I}\left [ \mathbf{E} + \frac{1}{\alpha}diag\left ( \left ( \mathbf{Z}_{i}-\mathbf{m} \right ) \left ( \mathbf{Z_{i}}-\mathbf{m} \right )^{T} \right ) \right ]^{-\frac{\alpha + 1}{2}}\mathbf{I}^T},
    \label{con:co-association}
\end{equation}
where $\mathbf{q_i} \in R^{1\times k}$, $\mathbf{Z_{i}}=\mathbf{I}^T \cdot \mathbf{z_{i}}$ 
represents the dimension broadcasting of the embedding representation for the $i^{th}$ 
sample $x_i$, and $\mathbf{E}$ is an identity matrix.

Then, the determinacy among sample $x_{i}$ and all centroids is determined by a mapping function 
as follow:
\begin{equation}
    \resizebox{.99\linewidth}{!}{$
        \begin{aligned}
            \mathbf{fq_i} = & \frac{\mathbf{Q_i} \cdot diag\left ( \mathbf{Q_i} \right ) }{t^{2}} \mathbf{1}\left ( \mathbf{Q_{ij}}<0  \right ) +  
            \frac{\mathbf{Q_i} \cdot diag\left ( \mathbf{Q_i} \right ) }{\left ( 1-t \right )^{2}} \mathbf{1}\left ( \mathbf{Q_{ij}}\ge 0 \right ),
        \end{aligned}
    $}
    \label{con:determinacy}
\end{equation}
here, $t$ represents the co-association probability at the lowest level of determinacy, 
which is adaptively determined using Otsu's method, 
$\mathbf{Q_{i}}=\mathbf{q_{i}}-t, \mathbf{Q_{i}} \in R^{1\times k} $ indicates the offset of 
$\mathbf{q_{i}}$ with respect to the threshold $t$, $diag$ 
denotes the construction of a diagonal matrix from a vector, and $\mathbf{1}(\cdot)$ represents an 
indicator function that equals one only when a certain condition is satisfied, and zero otherwise.

After obtaining the determinacy relationship between each sample-center pair, the stability of 
sample $x_{i}$ can be calculated based on the following formula:
\begin{equation}
    \resizebox{.99\linewidth}{!}{$
        \begin{aligned}
            \mathbf{sq_{i}} = \frac{1}{k} \mathbf{fq_i} \cdot \mathbf{I}^T - \frac{\lambda }{k} tr \left [ \left ( \mathbf{fq_i}-\frac{1}{k} \mathbf{fq_i}\cdot \mathbf{I}^T \mathbf{I} \right )^{T} \left ( \mathbf{fq_i}-\frac{1}{k} \mathbf{fq_i}\cdot \mathbf{I}^T \mathbf{I} \right ) \right ], 
        \end{aligned}
    $}
    \label{con:sq_i}
\end{equation}
here, $k$ represents the number of clusters, $\lambda$ is a proportionality coefficient, 
$tr$ denotes the trace of a matrix.
That is, the first and second terms of Eq. (\ref{con:sq_i}) represent the mean and 
variance of $\mathbf{fq_i}$, respectively.

Based on the above discussion, we can obtain the stability of each sample. During the 
process of clustering, we utilize instability as the loss and optimize 
the network parameters by minimizing this loss. Thus, our clustering loss function can be 
formulated as:
\begin{equation}
    L_{c} = 1-\frac{1}{n}\mathbf{I}\cdot \mathbf{sq},
    \label{con:Lc}
\end{equation}
where $n$ denotes the number of samples, 
$\mathbf{sq} = \left \{ \mathbf{sq_{i}} \right \}_{i=1}^{n}$, $\mathbf{sq} \in R^{n\times 1}$ 
represents the stability of the $n$ samples.

By optimizing this objective function, we can gradually move each sample close to its 
corresponding cluster and farther away from other clusters. 
This results in a high level of stability of all samples, close to one.
The training steps of the proposed DECS are shown in Algorithm \ref{alg:DECS}.

\subsection{Optimization and Convergence Analysis}
At each epoch, our model jointly optimizes the cluster centers $\left \{ m_{j} \right \}$ 
and neural network parameters $\theta$ using stochastic gradient descent with 
momentum. 
Firstly, the gradients of $L_c$ with respect to $\mathbf{sq}$ can be expressed as follows: 
\begin{equation}
    \frac{\partial L_{c}}{\partial \mathbf{sq}} = -\frac{1}{n}\mathbf{I}^T.
\end{equation}

Secondly, the gradient of stability $\mathbf{sq_i}$ of the $i^{th}$ sample with respect to 
deterministic $\mathbf{fq_i}$ can be written as:
\begin{equation}
    \frac{\partial \mathbf{sq_{i}}}{\partial \mathbf{fq_i}} = \frac{1}{k}\mathbf{I}^T - \frac{2\lambda }{k} \left [ \left ( \mathbf{fq_i}^{T}-\frac{1}{k} \mathbf{fq_i}^{T}\cdot \mathbf{I}\mathbf{I}^T \right ) \right ].
    \label{con:sq_to_fq}
\end{equation}

Thirdly, the gradient of determinacy $fq\left ( q_{ij}  \right )$ with respect to $q_{ij}$ can 
be written as follow: 
\begin{equation}
    \begin{aligned}
        \frac{\partial fq\left ( q_{ij}  \right )}{\partial q_{ij}} & = 
        \frac{2\left ( {q_{ij} - t} \right )}{t^{2} } \mathbf{1} \left ( q_{ij}<t \right ) \\
            & \quad + \frac{2\left ( {q_{ij} - t} \right )}{\left ( 1-t \right ) ^{2} }\mathbf{1} \left ( q_{ij}\ge t \right ),
    \end{aligned}
    \end{equation}
and thus the gradient of determinacy $\mathbf{fq_i}$ with respect to $\mathbf{q_i}$ can be written as
\begin{equation}
    \frac{\partial \mathbf{fq_i}}{\partial \mathbf{q_i}} = \left [ \frac{\partial fq\left ( q_{i1}  \right )}{\partial q_{i1}},\frac{\partial fq\left ( q_{i2}  \right )}{\partial q_{i2}},\dots,\frac{\partial fq\left ( q_{ik}  \right )}{\partial q_{ik}} \right ], 
\end{equation}
where $q_{ij}$ and $fq\left ( q_{ij}  \right )$ represents the $j^{th}$ element of 
$\mathbf{q_i}$ and $\mathbf{fq_i}$, respectively.

Lastly, for simplicity, we let $\alpha=1$ and 
\begin{equation}
    \begin{aligned}
        \mathbf{A} &:= diag\left ( \mathbf{E} + \mathbf{Z} \mathbf{Z}^{T} - 2\mathbf{Z}\mathbf{m}^{T} + \mathbf{m}\mathbf{m}^{T} \right ); \\ 
        \mathbf{B} &:= -diag\left ( 2\mathbf{Z}-2\mathbf{m} \right ); \\
        \mathbf{C} &:= -diag\left ( 2\mathbf{m}-2\mathbf{Z} \right ). 
        \label{con:temp}
    \end{aligned}
\end{equation}
The gradient of $\mathbf{q_i}$ with respect to $\mathbf{z_i}$ and $\mathbf{m}$ 
can be expressed separately as: 
\begin{equation}
    \frac{\partial \mathbf{q_i}}{\partial \mathbf{z_i}} = \frac{\mathbf{A}^{-2} \mathbf{B} \mathbf{I}^T \mathbf{I} \mathbf{A}^{-1} \mathbf{I}^T - \mathbf{A}^{-1} \mathbf{I}^T \mathbf{I} \mathbf{A}^{-2} \mathbf{B} \mathbf{I}^T}{\mathbf{I} \mathbf{A}^{-1} \mathbf{I}^T\mathbf{I} \mathbf{A}^{-1} \mathbf{I}^T }, 
\end{equation}
\begin{equation}
    \frac{\partial \mathbf{q_i}}{\partial \mathbf{m}} = \frac{\mathbf{A}^{-2} \mathbf{C} \mathbf{I}^T \mathbf{I} \mathbf{A}^{-1} \mathbf{I}^T - \mathbf{A}^{-1} \mathbf{I}^T \mathbf{I} \mathbf{A}^{-2} \mathbf{C} \mathbf{I}^T}{\mathbf{I} \mathbf{A}^{-1} \mathbf{I}^T\mathbf{I} \mathbf{A}^{-1} \mathbf{I}^T }. 
\end{equation}

\begin{algorithm}[tb]
    \caption{algorithm of DECS}
    \label{alg:DECS}
    \textbf{Input}: Dataset X, Number of cluster K, Maximum iterations MaxIter;\\
    \textbf{Output}: Cluster center $m$, Cluster assignment $s$;

    \begin{algorithmic}[1] 
        \STATE Initialize autoencoder’s weight by (\ref{con:Lr}) with X;
        \STATE Initialize $m$ and $s$ with k-means algorithm;
        \WHILE{iter $\le$ MaxIter}
            \STATE compute the co-association probability matrix of samples $x\in X$ and centers $m$ by (\ref{con:co-association});
            \STATE compute the determinacy between samples $x\in X$ and centers $m$ by (\ref{con:determinacy}); 
            \STATE compute the stability $sq$ of n samples $x\in X$ by (\ref{con:sq_i});
            \STATE update encoder's weight and $m$ by optimizing (\ref{con:Lc_z})(\ref{con:Lc_m});
        \ENDWHILE
        \STATE \textbf{return} Cluster center $m$, Cluster assignment $s$ by maximizing $sq$.
    \end{algorithmic}
\end{algorithm}


Therefore, the gradients of $L_c$ with respect to the latent embedding $\mathbf{z_i}$ 
and cluster centroid $\mathbf{m}$ are be computed as:
\begin{equation}
    \frac{\partial L_{c} }{\partial \mathbf{z_i}} =\frac{\partial L_{c} }{\partial \mathbf{sq_i}} \cdot \frac{\partial \mathbf{sq_i}}{\partial \mathbf{fq_i}} \cdot \frac{\partial \mathbf{fq_i}}{\partial \mathbf{q_i}} \cdot \frac{\partial \mathbf{q_i}}{\partial \mathbf{z_i}},
    \label{con:Lc_z}
\end{equation}
\begin{equation}
    \frac{\partial L_{c} }{\partial \mathbf{m}} =\frac{\partial L_{c} }{\partial \mathbf{sq_i}} \cdot \frac{\partial \mathbf{sq_i}}{\partial \mathbf{fq_i}} \cdot \frac{\partial \mathbf{fq_i}}{\partial \mathbf{q_i}} \cdot \frac{\partial \mathbf{q_i}}{\partial \mathbf{m}}.
    \label{con:Lc_m}
\end{equation}

Then, the gradients $\frac{\partial L_{c} }{\partial z_i}$ are propagated to the 
neural network and used in backpropagation to compute the network's parameter gradient 
$\frac{\partial L_{c} }{\partial \theta}$. 
Through iterating these updates, the model finds the optimal clustering result. 
The training process is repeated until the convergence condition is met.

To validate the correctness of the optimization process, Fig.\ref{fig:formula} presents the 
graphs of functions and their derivatives involved in computing sample stability for the case of 
two classes. 
\begin{figure}[htb]
    \centering
    \begin{minipage}[b]{0.21\textwidth}
        \centering
        \includegraphics[width=\textwidth]{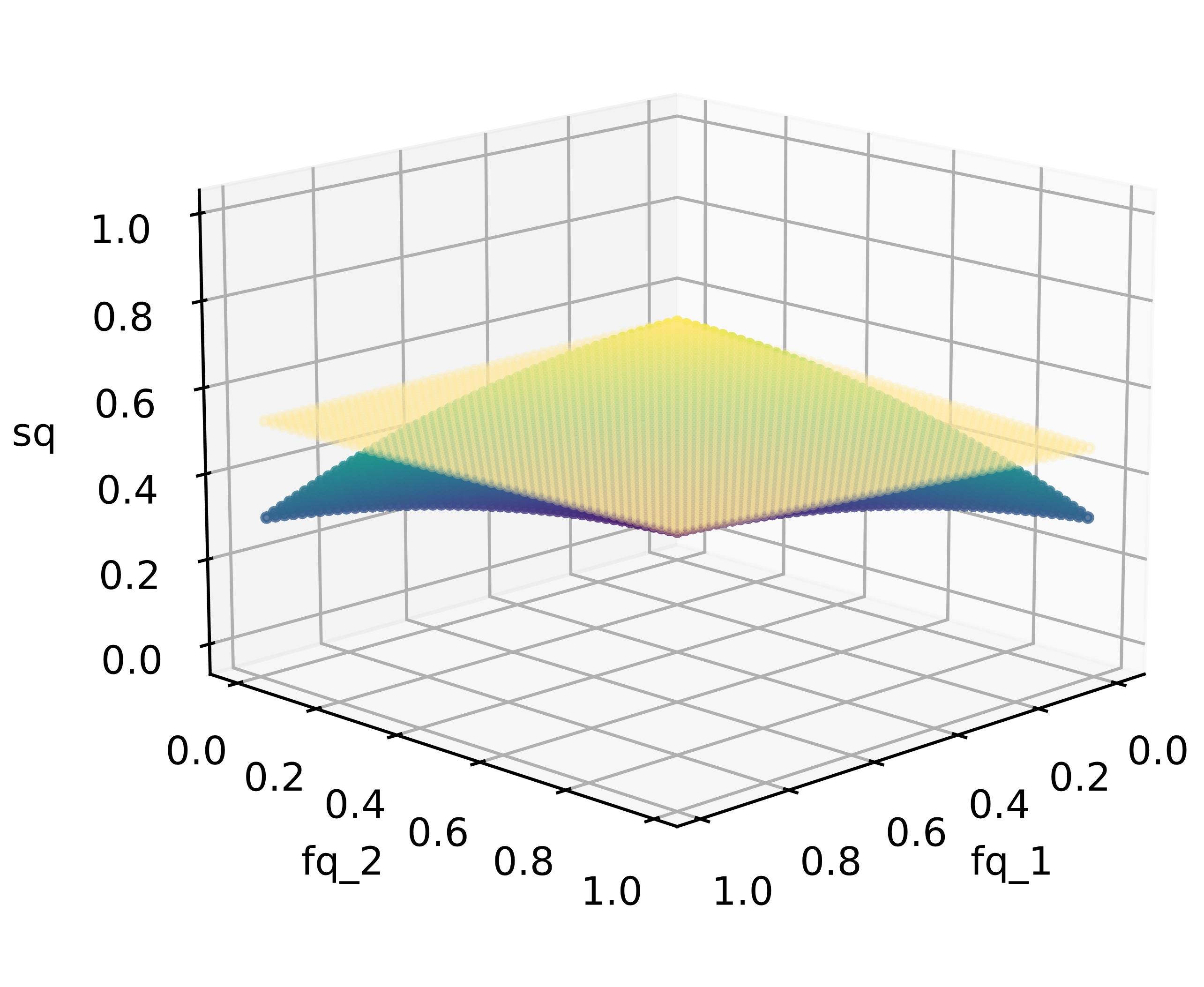}
        \makebox[0pt][c]{(a)}
    \end{minipage}\hfill
    \begin{minipage}[b]{0.21\textwidth}
        \centering
        \includegraphics[width=\textwidth]{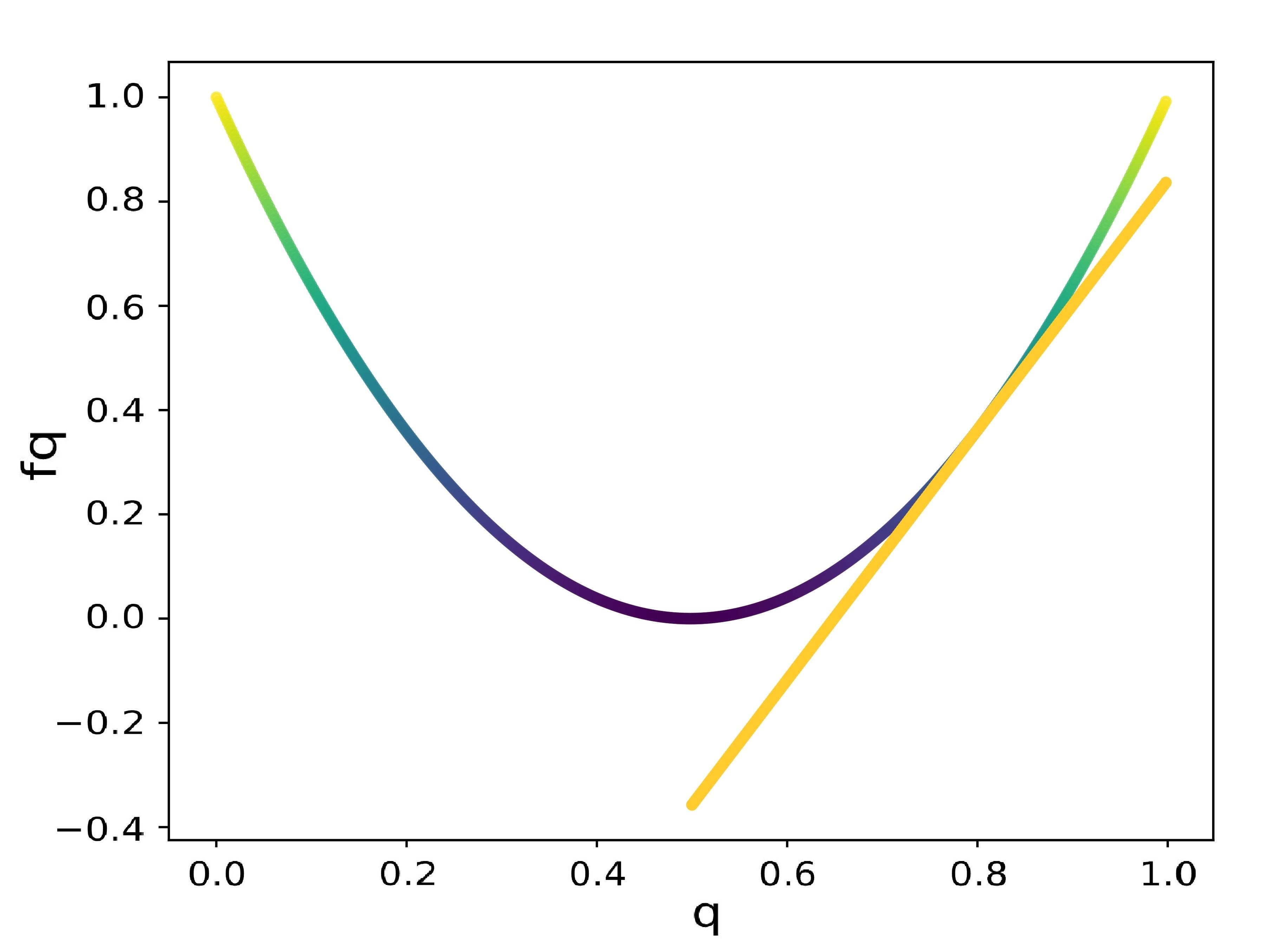}
        \makebox[0pt][c]{(b)}
    \end{minipage}\hfill
    \begin{minipage}[b]{0.21\textwidth}
        \centering
        \includegraphics[width=\textwidth]{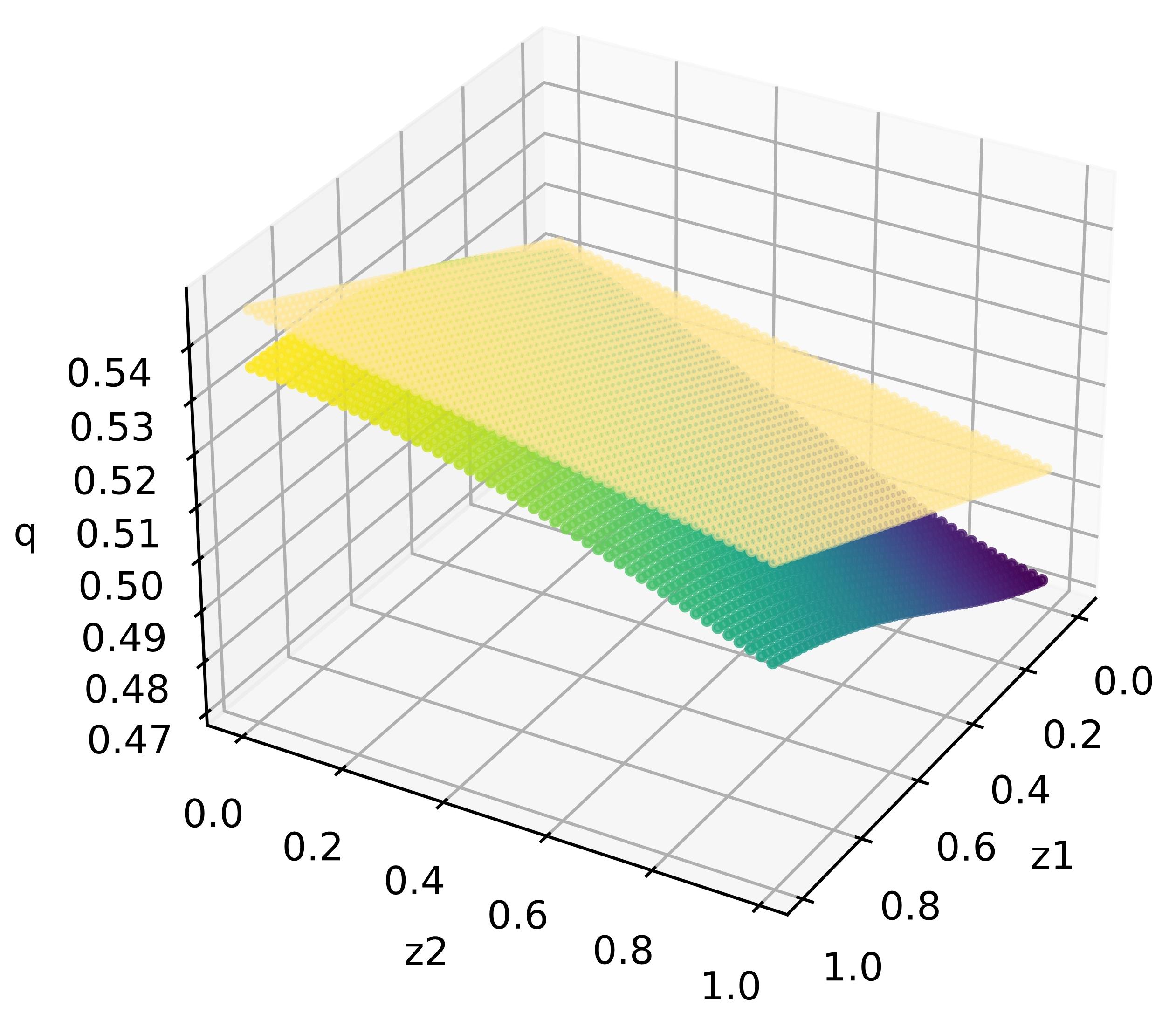}
        \makebox[0pt][c]{(c)}
    \end{minipage}\hfill
    \begin{minipage}[b]{0.21\textwidth}
        \centering
        \includegraphics[width=\textwidth]{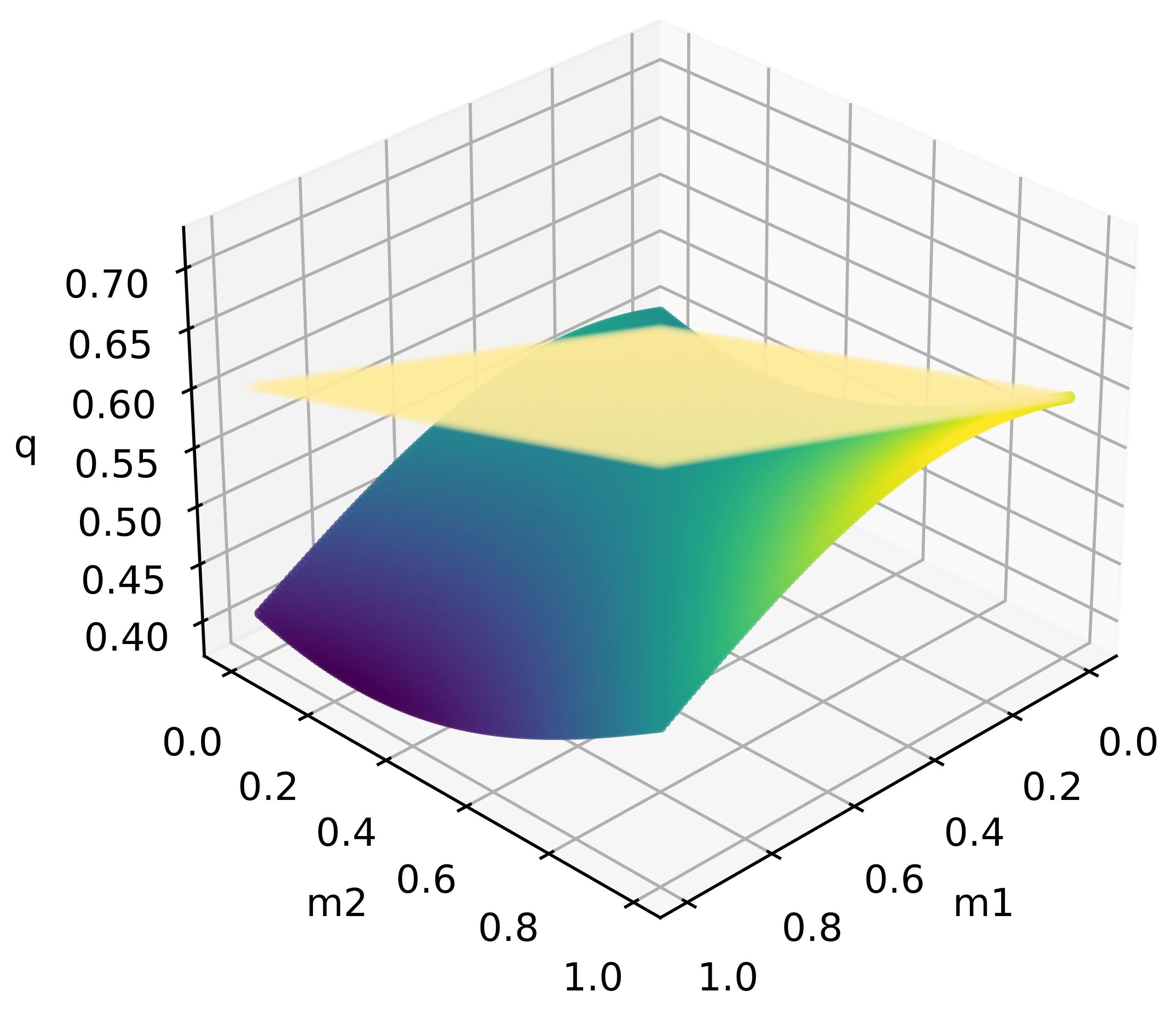}
        \makebox[0pt][c]{(d)}
    \end{minipage}
    
    \caption{Visualization of the functions and their derivatives involved in sample stability clustering.
            (a) shows the function of $sq$ w.r.t $fq$ and its derivative in the case of two centroids. 
            (b) represents the function of $fq$ w.r.t $q$ and its derivative. 
            (c) demonstrates the function of $q$ w.r.t two dimensions vector $z$ and its derivative. 
            (d) depicts the function of $q$ w.r.t two cluster centers $m$ and its derivative. 
            }
    \label{fig:formula}
\end{figure}

Furthermore, We have theoretically proven the convergence of the loss 
$L_{c}$ with respect to the centroids $\mathbf{m}$. 

\begin{theorem}
    There exists $M > 0$ such that $||\bigtriangledown L_{c}|| \le M$, 
    where $M = \frac{2\left ( 1+2\lambda  \right ) \left ( \alpha +1 \right ) }{4nkt^{2} \alpha }  max(\left| \left| z_{i}-m_{j} \right| \right|)$.
\end{theorem}

\begin{proof}
    \begin{lemma}
        \cite{nesterov1998introductory} Let $f$ be Lipschitz continuous on the ball 
        $B_{2} \left ( x^{*} , R \right )$ with the constant M and $||x_{0}-x^{*}||\le R$. Then
        \begin{equation}
            f_{k} ^{*} - f^{*} \le M \frac{R^{2} +\sum_{i=0}^{k} h_{i}^{2}}{2\sum_{i=0}^{k} h_{i}},
        \end{equation}
        where $h_{i}$ represents the step size and $M$ is termed as the Lipschitz 
        constant. 
    \end{lemma}

    According to Lemma 1, its convergence can be determined by the initial solution and step size 
    when the loss function satisfies Lipschitz continuity.
    Consequently, we can conclude based on Lemma 1 that the objective function $L_{c}$ is convergence 
    $i.f.f.$ $||\bigtriangledown L_{c} || \le M$. In other words, 
    to verify the convergence of $L_c$, it is necessary to prove the existence of an upper bound for 
    its derivative.

    For the sake of simplicity, a centroid $m_j$ is used as an example for the convergence 
    analysis, and we only consider the case where $q_{ij} < t$, while the case of $q_{ij} \ge t$ can be 
    similarly treated. We can know:
    \begin{equation}
            \left| \left|
                    \frac{\partial L_c}{\partial m_j}
            \right| \right| = 
            \left| \left|
                \frac{\partial L_c}{\partial sq(x_{i})} \frac{\partial sq(x_{i})}{\partial fq(q_{ij})} \frac{\partial fq(q_{ij})}{\partial q_{ij}} \frac{\partial q_{ij}}{\partial m_{j}} 
            \right| \right|, 
        \label{con:convergence}
    \end{equation}
    where $sq(x_{i})$, $fq(q_{ij})$ and $q_{ij}$ represents the stability of $i^{th}$ sample, 
    determinacy and co-association probability of the $i^{th}$ sample with respect to the $j^{th}$ 
    centroid. 

    For the first term pf the Eq.(\ref{con:convergence}), it is clear that  
    $\frac{\partial L_c}{\partial sq(x_{i})} = -\frac{1}{n}$, 
    and as for the second item, we know that:
    \begin{equation}
        \left| \left| \frac{\partial sq(x_{i})}{\partial fq(q_{ij})} \right| \right|= \frac{1}{k} - \frac{2\lambda }{k} \left ( fq(q_{ij}) -\mu \right )
        \le \frac{1 + 2\lambda }{k}, 
    \end{equation}
    here, $\mu$ represents the mean of $\left \{ fq(q_{ij}) \right \}_{j=1}^{k} $ and $0 \le fq(q_{ij}) \le 1$. 

    For the third term of the Eq.(\ref{con:convergence}), in the case of $q_{ij} < t$: 
    \begin{equation}
            \left| \left| \frac{\partial fq(q_{ij})}{\partial q_{ij}} \right| \right| = \frac{2\left ( q_{ij} -t \right ) }{t^2}  
            \ge \frac{-2}{t^2}, 
    \end{equation}
    and for the last term of the Eq.(\ref{con:convergence}), we know that: 
    \begin{equation}
        \begin{aligned}
            \left| \left| \frac{\partial q_{ij}}{\partial m_{j}} \right| \right| &= \left| \left| \frac{\alpha +1}{\alpha } \frac{z_{i} -m_{j}}{1+||z_{i}-m_{j}||^2 / \alpha } q_{ij}\left ( 1-q_{ij}  \right ) \right| \right| \\ 
            &\le \frac{\alpha +1}{4\alpha } \left| \left| \left ( z_{i}-m_{j} \right ) \right| \right|,
        \end{aligned}
    \end{equation}
    where, $z_{i}$ represents the $i^{th}$ sample embedding, $1+||z_{i}-m_{j}||^2 / \alpha \ge 1$, 
    and $q_{ij}\left ( 1-q_{ij} \right ) \le \frac{1}{4} $ due to $0 \le q_{ij} \le 1$. 

    According to the above analysis, we can conclude that the upper bound for the loss $L_{c}$:
    \begin{equation}
        \left| \left|
                \frac{\partial L_c}{\partial m_j} 
        \right| \right| \le 
        \frac{2\left ( 1+2\lambda  \right ) \left ( \alpha +1 \right ) }{4nkt^{2} \alpha } \left| \left| z_{i}-m_{j} \right| \right|.
    \end{equation}

    That is to say, there exists $M > 0$ such that $||\bigtriangledown L_{c}|| \le M$, 
    where $M = \frac{2\left ( 1+2\lambda  \right ) \left ( \alpha +1 \right ) }{4nkt^{2} \alpha }  max(\left| \left| z_{i}-m_{j} \right| \right|)$.
    In fact, there exists an upper boundary of $ \left| \left| z_{i}-m_{j} \right| \right|$ for 
    any real-world dataset.
\end{proof}

\section{Experiments}
In this section, we evaluate the effectiveness of the proposed DECS method on five 
benchmark datasets. We also present the visualization of sample distribution and analyze 
how these hyperparameters impact the performance.

\subsection{Datasets and Evaluation Metrics}
In order to validate the performance and generality of the proposed method, we perform 
experiments on five image datasets, as shown in Table \ref{tab:dataset}. Considering that 
clustering tasks are fully unsupervised, the training and test split are merged in 
all our experiments. 

\begin{table}[htbp] 
    \centering
    \begin{tabular}{cccc}
        \hline
        {Dataset}  & {samples} & {Classes} & {Dimensions} \\
        \hline
        {MNIST-full} & {70,000} & {10} & {1x28x28} \\
        {MNIST-test} & {10,000} & {10} & {1x28x28} \\
        {USPS} & {9,298} & {10} & {1x16x16} \\
        {Fashion-Mnist} & {70,000} & {10} & {1x28x28} \\
        {YTF} & {12,183} & {41} & {3x55x55} \\
        \hline
    \end{tabular}
    \caption{Description of Datasets}
    \label{tab:dataset}
\end{table}

Two widely-used unsupervised evaluation metrics including cluster accuracy(ACC) 
and normalized mutual information(NMI) are used to validate the performance of 
the proposed model. 
Higher values of these metrics indicate better clustering performance. 

\subsection{Baseline Methods}
In the comparative experiments, our proposed method was compared with several representative 
conventional baseline, including:
k-means \cite{macqueen1967classification}, 
GMM \cite{reynolds2009gaussian}, 
sc-Ncut \cite{shi2000normalized}, 
SC-LS \cite{chen2011large} and 
In addition, our method was compared with several state-of-the-art deep clustering algorithms. 
To ensure the fairness of our experiments, we have selected comparative methods that utilize 
autoencoders as feature extractors, including: 
DEC \cite{xie2016unsupervised}, 
IDEC \cite{guo2017improved}, 
VaDE \cite{jiang2016variational}, 
JULE \cite{yang2016joint}, 
DEPICT \cite{ghasedi2017deep}, 
IDCEC \cite{lu2022improved}, 
TELL \cite{peng2022xai}, 
AdaGAE \cite{li2021adaptive}, 
MI-ADM \cite{jabi2019deep}, 
DSCDA \cite{yang2019deep}, 
DynAE \cite{mrabah2020deep}, 
ASPC-DA \cite{guo2019adaptive}, 
TDEC \cite{zhang2023tdec} and 
DeepDPM \cite{ronen2022deepdpm}.

\subsection{Experiment Results}
Table \ref{tab1} presents the performance of our method and other comparative 
methods. 
For the compared methods, if their results on some datasets were not reported, 
we ran the released code with hyperparameters mentioned in their 
papers, and the results are marked by (*) on top. When the code is not publicly 
available or running the released code is not practical, we replaced the corresponding 
results with dashes (-).

\begin{table*}[htbp] 
    \centering
        \resizebox{0.85\textwidth}{!}{
            \begin{tabular}{c|c c|c c|c c|c c|c c}
            \hline
            \multirow{2}{*}{methods} & \multicolumn{2}{ c|}{MNIST} & \multicolumn{2}{ c|}{MNIST-test} 
                                            & \multicolumn{2}{ c|}{USPS} & \multicolumn{2}{ c }{Fashion} 
                                            & \multicolumn{2}{ c }{YTF}\\
            \cline{2-11}
            \textbf{} & \textit{ACC}& \textit{NMI} 
                    & \textit{ACC}& \textit{NMI}
                    & \textit{ACC}& \textit{NMI}
                    & \textit{ACC}& \textit{NMI}
                    & \textit{ACC}& \textit{NMI}\\
            \hline
            k-means & {0.532} & {0.499}
            & {0.542}& {0.500}
            & {0.668}& {0.626} 
            & {0.474}& {0.512}
            & {0.601}& {0.776}\\
            {GMM} & {0.433}& {0.366}
            & {0.540}& {0.593}
            & {0.551}& {0.530}
            & {0.556}& {0.557}
            & {0.348}& {0.411}\\
            {SC-Ncut} & {0.656}& {0.731}
            & {0.660}& {0.704}
            & {0.649}& {0.794}
            & {0.508}& {0.575}
            & {0.510}& {0.701}\\
            {SC-LS} & {0.714}& {0.706}
            & {0.740}& {0.756}
            & {0.746}& {0.755}
            & {0.496}& {0.497}
            & {0.544}& {0.759}\\
            \hline
            {DEC} & {0.863}& {0.834}
            & {0.856}& {0.830} 
            & {0.762}& {0.767}
            & {0.518}& {0.546}
            & {0.371}& {0.446}\\
            {IDEC} & {0.881*}& {0.867*}
            & {0.846}& {0.802}
            & {0.761*}& {0.785*}
            & {0.529}& {0.557}
            & {0.400*}& {0.483*}\\
            {VaDE} & {0.944}& {0.891}
            & {0.944*}& {0.885*} 
            & {0.566*}& {0.512*} 
            & {0.629}& {0.611}
            & {0.601*}& {0.753*}\\
            {JULE} & {0.964}& {0.913}
            & {0.961}& {0.915} 
            & {0.950}& {0.913}
            & {0.563*}& {0.608*}
            & {0.684}& {0.848}\\
            {DEPICT} & {0.965*}& {0.917*}
            & {0.963*}& {0.915*}
            & {0.899}& {0.906}
            & {0.392}& {0.392}
            & {0.621}& {0.802}\\
            {IDCEC} & {0.948}& {0.906}
            & {0.923}& {0.853} 
            & {0.812}& {0.858}
            & {-}& {-}
            & {0.632}& {0.793}\\
            {TELL} & {0.952}& {0.888} 
            & {0.776*}& {0.751*}
            & {0.865*}& {0.786*} 
            & {0.584*}& {0.658*}
            & {-}& {-}\\
            {AdaGAE} & {0.929}& {0.853} 
            & {-}& {-} 
            & {0.920}& {0.848} 
            & {-}& {-}
            & {-}& {-}\\
            {MI-ADM} & {0.969}& {0.922}
            & {0.871}& {0.885}
            & {0.979}& {0.948}
            & {-}& {-}
            & {0.606}& {0.801}\\
            {DSCDA} & {0.978}& {0.941} 
            & {0.980}& {0.946} 
            & {0.869}& {0.857}
            & {\textbf{0.662}}& {0.645}
            & {0.691}& {0.857}\\
            {DynAE} & {0.987}& {0.964} 
            & {\textbf{0.987}}& {\textbf{0.963}} 
            & {0.981}& {0.948}
            & {0.591}& {0.642}
            & {-}& {-}\\
            {ASPC-DA} & {\textbf{0.988}}& {\textbf{0.966}} 
            & {0.973}& {0.936} 
            & {\textbf{0.982}}& {\textbf{0.951}} 
            & {0.591}& {0.654}
            & {-}& {-}\\
            {DeepDPM} & {0.980}& {0.950}
            & {-}& {-} 
            & {0.940}& {0.900} 
            & {0.610}& {0.500}
            & {0.821}& {\textbf{0.930}}\\
            {TDEC} & {0.985}& {0.957} 
            & {0.975}& {0.935} 
            & {0.976}& {0.935}
            & {\textbf{0.645}}& {\textbf{0.693}}
            & {\textbf{0.950}}& {\textbf{0.980}}\\
            {\textbf{DECS}} & {\textbf{0.990}}& {\textbf{0.973}}
            & {\textbf{0.990}}& {\textbf{0.971}}
            & {\textbf{0.992}}& {\textbf{0.976}}
            & {0.642}& {\textbf{0.716}}
            & {\textbf{0.827}}& {0.911}\\
            \hline
        \end{tabular}
        }
        \caption{Comparison of clustering performance on five datasets in terms of ACC and NMI.
                The bolded font represents the best and second results.}
        \label{tab1}
\end{table*}

\begin{figure}[htbp]
    \centering
    \begin{minipage}[b]{0.15\textwidth}
        \centering
        \includegraphics[width=\textwidth]{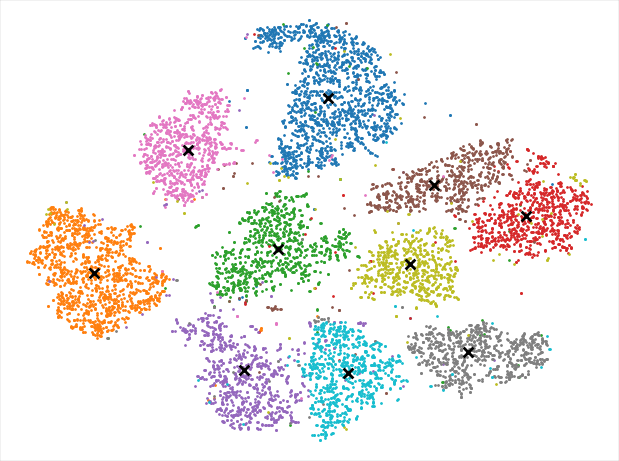}\quad
        \makebox[0pt][c]{(a) DEC}
    \end{minipage}\hfill
    \begin{minipage}[b]{0.15\textwidth}
        \centering
        \includegraphics[width=\textwidth]{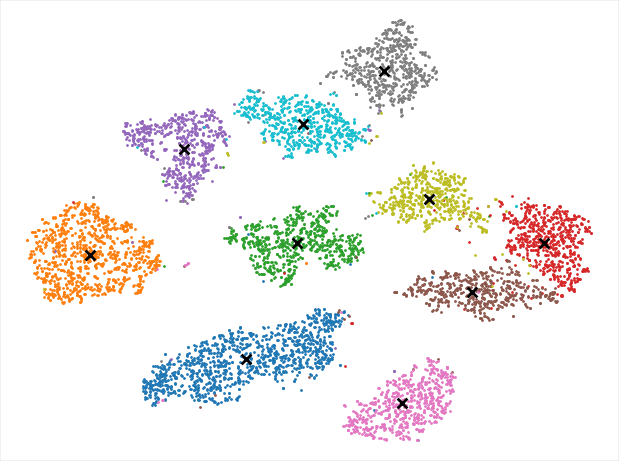}\quad
        \makebox[0pt][c]{(b) TELL}
    \end{minipage}\hfill
    \begin{minipage}[b]{0.15\textwidth}
        \centering
        \includegraphics[width=\textwidth]{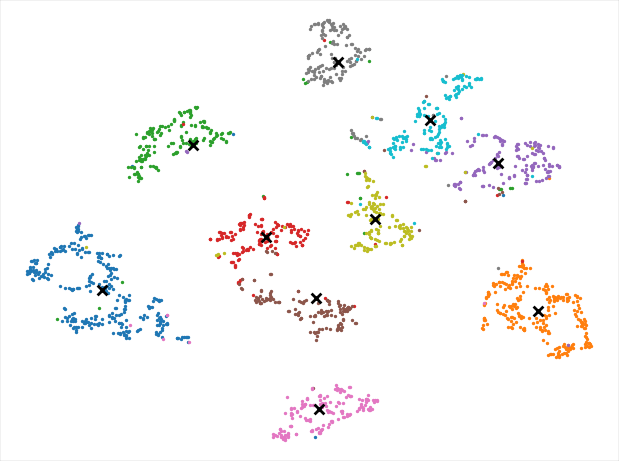}\quad
        \makebox[0pt][c]{(c) AdaGAE}
    \end{minipage}\hfill
    \begin{minipage}[b]{0.15\textwidth}
        \centering
        \includegraphics[width=\textwidth]{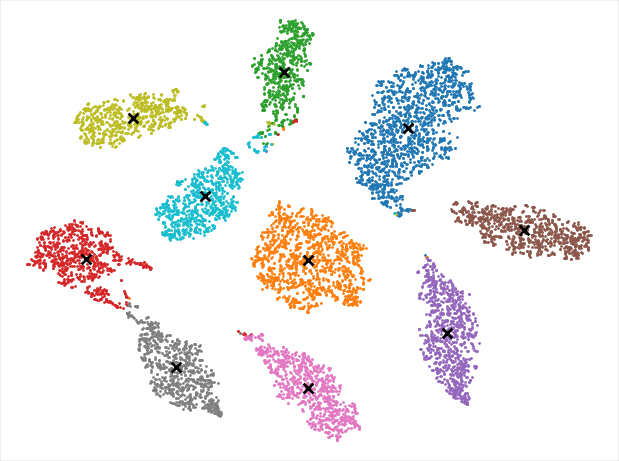}\quad
        \makebox[0pt][c]{(d) ASPC-DA}
    \end{minipage}\hfill
    \begin{minipage}[b]{0.15\textwidth}
        \centering
        \includegraphics[width=\textwidth]{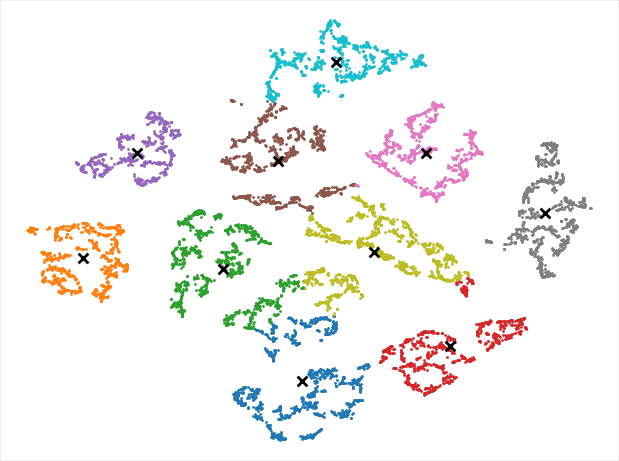}\quad
        \makebox[0pt][c]{(e) DeepDPM}
    \end{minipage}
    \begin{minipage}[b]{0.15\textwidth}
        \centering
        \includegraphics[width=\textwidth]{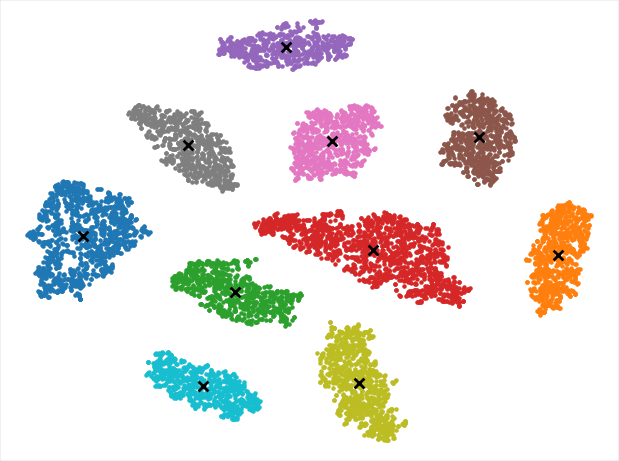}\quad
        \makebox[0pt][c]{(f) Ours}
    \end{minipage}\hfill
    \caption{T-SNE visualization comparing the cluster structures obtained from different 
            clustering algorithms on the USPS datasets.
            Distinct colors represent different digits,
            and the cluster centers are indicated by black 'x' symbols.}
    \label{fig:distribtuion_compare}
\end{figure}

It be seen from Table \ref{tab1} that the proposed DECS algorithm achieves superior 
clustering results across all datasets. 
Specifically, on the USPS dataset, our algorithm achieves a clustering 
accuracy of over 99\%.  It outperforms the second-best ASPC-DA by 1.0\% and 2.5\% on 
ACC and NMI, respectively. Furthermore, our method significantly outperforms several 
classical shallow baselines, which can be attributed to the utilization of a multi-layer 
convolutional autoencoder as the feature extractor. 

Furthermore, we also performed t-SNE visualization to compare the cluster structures obtained 
using different clustering algorithms on the USPS dataset, as shown in Fig.\ref{fig:distribtuion_compare}.
Specifically, Fig.\ref{fig:distribtuion_compare} (a)-(e) represent the clustering results 
obtained by algorithms DEC, TELL, AdaGAE, ASPC-DA, and DeepDPM, respectively, while 
Fig.\ref{fig:distribtuion_compare} (f) represents the clustering result of our proposed 
algorithm. It is evident that our proposed algorithm is able to achieve clearer and more accurate 
cluster structures, which further proves the effectiveness of the proposed algorithm.

In addition, we investigated the sensitivity of our model to the parameters $\alpha$ and 
$\lambda$ using the USPS dataset, as shown in Fig.\ref{fig:parameters}, 
where Fig.\ref{fig:parameters}(a) displays the results of ACC from different 
parameter settings, and Fig.\ref{fig:parameters}(b) shows the results of NMI. 
We can observe from the figure that the clustering performance is not significantly 
affected by variations in the hyperparameters. 

\begin{figure}[htbp]
    \centering
    \begin{minipage}[b]{0.23\textwidth}
        \centering
        \includegraphics[width=\textwidth]{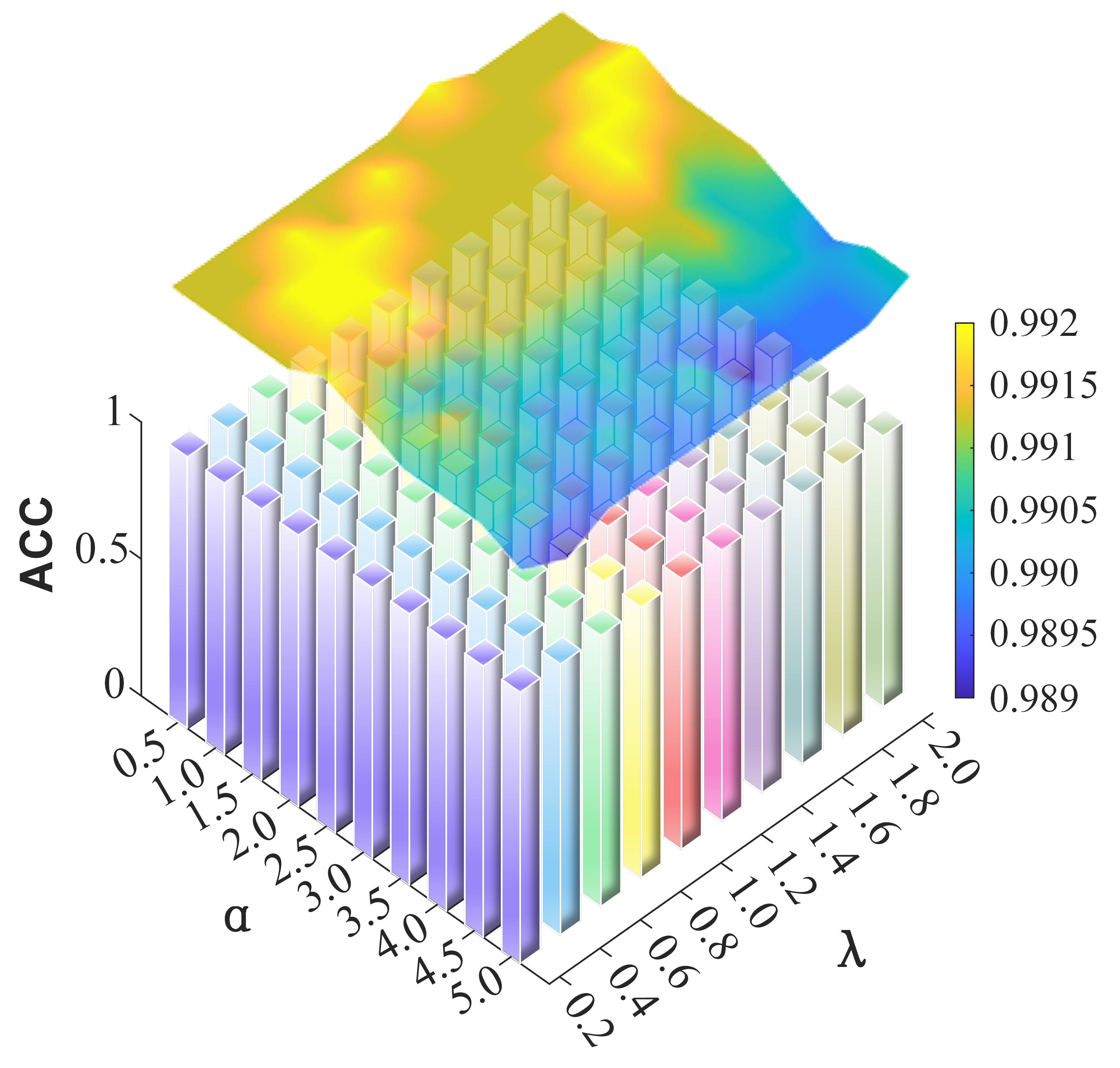} \qquad
        \makebox[0pt][c]{(a)}
    \end{minipage}
    \begin{minipage}[b]{0.23\textwidth}
        \centering
        \includegraphics[width=\textwidth]{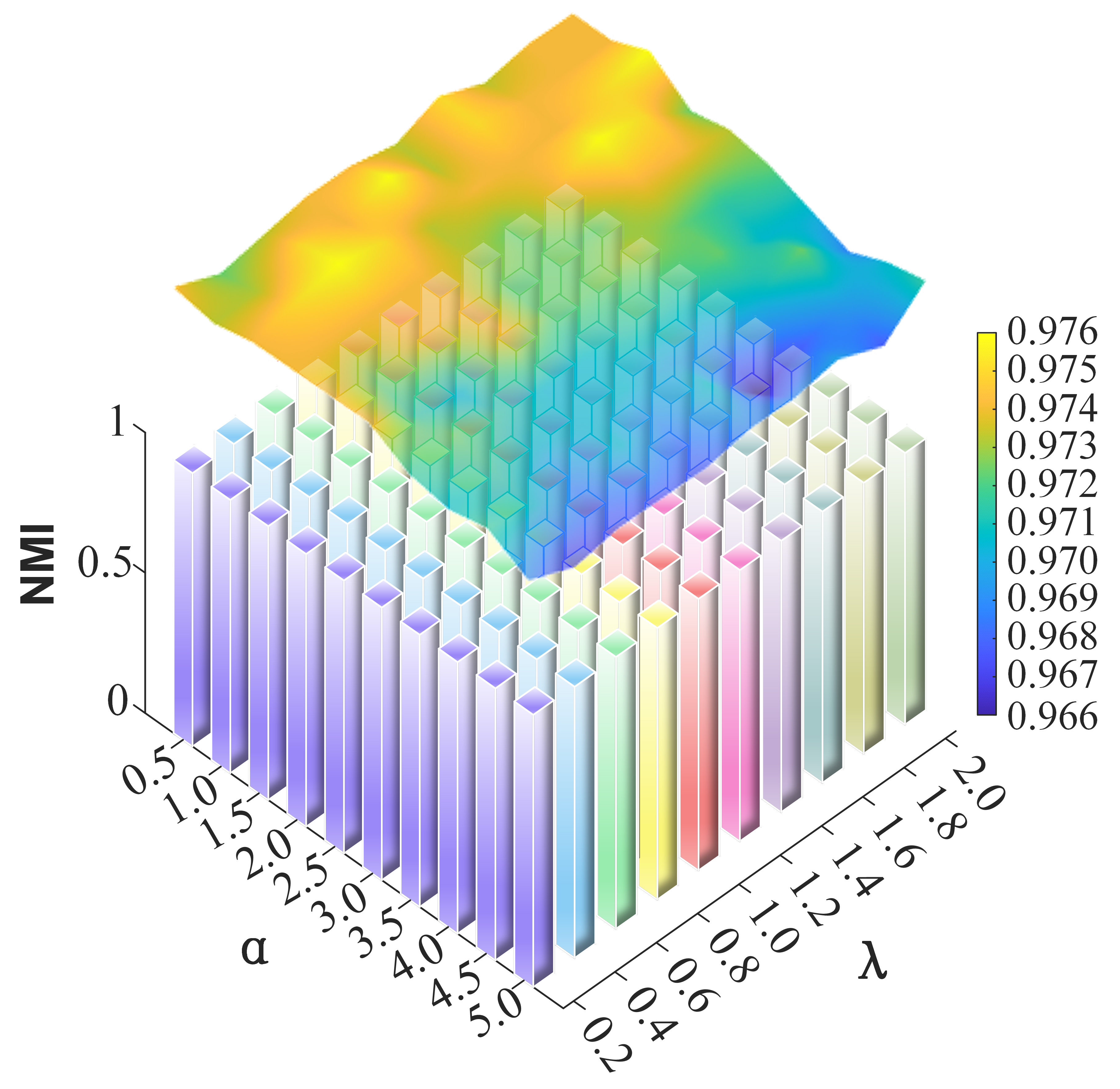}
        \makebox[0pt][c]{(b)}
    \end{minipage}
    \caption{ACC and NMI of our method with different $\alpha$ and $\lambda$ on USPS dataset.}
    \label{fig:parameters}
\end{figure}

\subsection{Visualization Analyses}
In this section, we conducted t-SNE visualization on the USPS dataset to give a more intuitive 
understanding of the proposed DECS algorithm. Specifically, we employ 
t-SNE to reduce the dimensionality of the learned sample representation to two dimensions 
and visualize the distribution of all samples along with their respective cluster centroids. 
As shown in Fig.\ref{fig:distribtuion}, Fig.\ref{fig:distribtuion}(a) to 
Fig.\ref{fig:distribtuion}(c) depict the sample visualization during the process of 
representation learning, and Fig.\ref{fig:distribtuion}(d) to 
Fig.\ref{fig:distribtuion}(f) illustrate the sample visualization during the process of 
clustering. 
It can be observed that, as the training progresses, DECS learns a more compact and 
discriminative representation, resulting in an improved separability of the estimated cluster 
centers.

\begin{figure}[htbp]
    \centering
    \begin{minipage}[b]{0.15\textwidth}
        \centering
        \includegraphics[width=\textwidth]{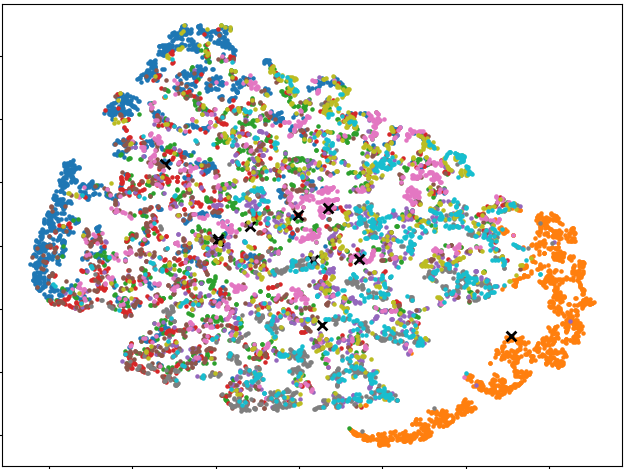}\quad
        \makebox[0pt][c]{(a)}
    \end{minipage}\hfill
    \begin{minipage}[b]{0.15\textwidth}
        \centering
        \includegraphics[width=\textwidth]{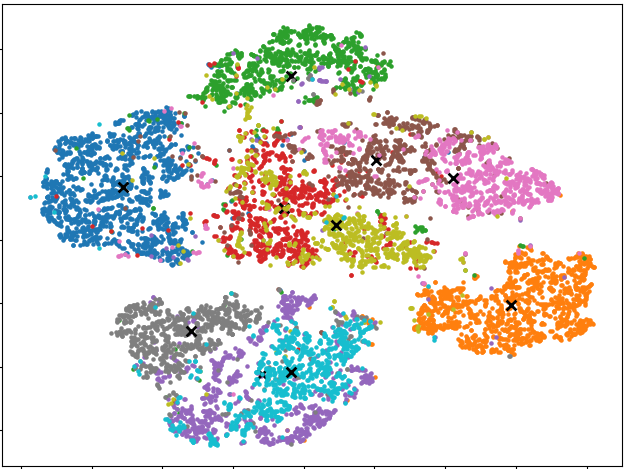}\quad
        \makebox[0pt][c]{(b)}
    \end{minipage}\hfill
    \begin{minipage}[b]{0.15\textwidth}
        \centering
        \includegraphics[width=\textwidth]{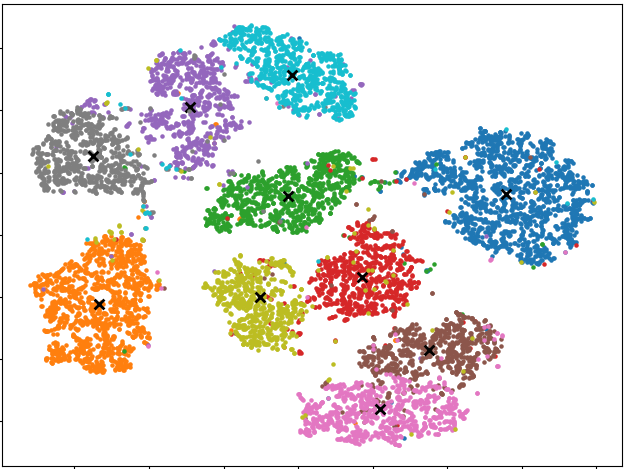}\quad
        \makebox[0pt][c]{(c)}
    \end{minipage}
    Sample distribution changes during the representation learning stage.

    \begin{minipage}[b]{0.15\textwidth}
        \centering
        \includegraphics[width=\textwidth]{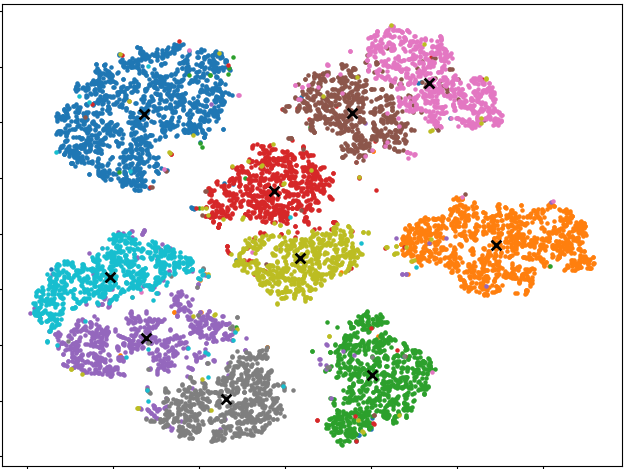}\quad
        \makebox[0pt][c]{(d)}
    \end{minipage}\hfill
    \begin{minipage}[b]{0.15\textwidth}
        \centering
        \includegraphics[width=\textwidth]{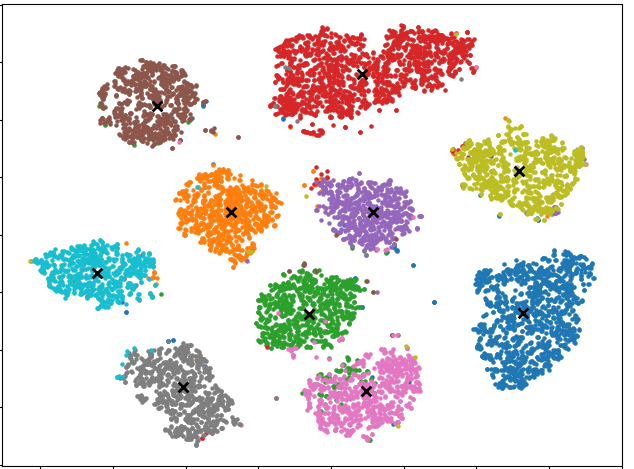}\quad
        \makebox[0pt][c]{(e)}
    \end{minipage}\hfill
    \begin{minipage}[b]{0.15\textwidth}
        \centering
        \includegraphics[width=\textwidth]{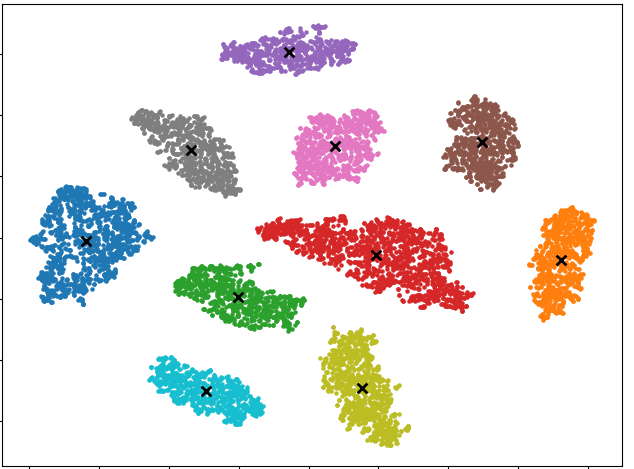}\quad
        \makebox[0pt][c]{(f)}
    \end{minipage}
    Sample distribution changes during the clustering stage.
    \caption{T-SNE visualization on the learned USPS representations across the 
            training process. Different digits are denoted with different colors,
            and the black symbol "x" denotes the cluster centers estimated by DECS.}
    \label{fig:distribtuion}
\end{figure}

\subsection{Implementation}
For all datasets, we specify that the encoder to consists of four convolutional layers with 
channel sizes of 32, 64, 128, and 256, respectively. Each convolutional kernel has a 
size of 3×3 and uses a stride of 2. 
Furthermore, batch normalization and max pooling layers are added after each 
convolutional layer. 
The decoder uses a network that mirrors the encoder's structure. 
Additionally, ReLU is utilized as the activation function for all convolutional layers 
in the model.

During the training process, data augmentation techniques such as random rotation, 
translation, and cropping are applied to improve the neural network's generalization ability. 
In addition, the autoencoder is trained end-to-end for 500 epochs using the Adam optimizer 
with default parameters in Keras. Then, the encoder is further trained for 10000 iterations 
with a batch size of 256. 
The coefficient $\lambda$ for variance is set to 0.8 during the calculation of 
sample stability.

\section{Conclusion}
In this paper, we proposed a deep embedding clustering algorithm driven by 
sample stability. 
The algorithm combines a convolutional autoencoder model with a clustering layer that 
relies on sample stability. 
Unlike previous methods, our method constrains the sample using sample stability, 
eliminating the need for artificially constructed pseudo targets. 
This mitigated the clustering biases caused by inappropriate pseudo targets and significantly 
improved the reliability of the clustering results. 
We analyzed the convergence of the proposed DECS
model, and the experimental results on five image datasets indicate that our algorithm achieve 
superior clustering performance. 
In the future, incorporating more complex representation learning models 
and applying our approach to a wider range of real-world datasets may be an intriguing 
and practical avenue for research.

\bibliographystyle{named}
\bibliography{Deep_Embedding_Clustering_Driven_by_Sample_Stability}

\end{document}